\documentclass[conference]{IEEEtran}
\usepackage{graphicx}
\usepackage{textcomp}
\usepackage{xcolor}
\usepackage{algorithmicx}
\usepackage{algorithm}
\usepackage{algpseudocode} 
\usepackage{multirow} 
\usepackage{booktabs} 
\usepackage{rotating} 
\usepackage{graphicx} 
\usepackage{array}    
\usepackage{amsthm}
\usepackage{amsmath}
\usepackage{amsfonts} 
\usepackage{amssymb}
\newtheorem{theorem}{Theorem}
\newcommand{\argmin}[1]{\underset{#1}{\operatorname{argmin}}~}
\newcommand{\argmax}[1]{\underset{#1}{\operatorname{argmax}}~}

\begin{document}

\title{CFSSeg: Closed-Form Solution for Class-Incremental Semantic Segmentation of 2D Images and 3D Point Clouds}

\author{
Jiaxu Li$^{1}$, 
Rui Li$^{1,5}$,
Jianyu Qi$^{1}$,
Songning Lai$^{2}$, 
Linpu Lv$^{6}$,\\
Kejia Fan$^{1}$, 
Jianheng Tang$^{4}$, 
Yutao Yue$^{2}$, 
Dongzhan Zhou$^{5}$,
Yuanhuai Liu$^{4}$,
Huiping Zhuang$^{3}$\\
$^1$Central South University, $^2$The Hong Kong University of Science and Technology (Guangzhou), \\
$^3$South China University of Technology, $^4$Peking University,\\ $^5$Shanghai Artificial Intelligence Laboratory,$^6$Zhengzhou University
}
\maketitle
\begin{abstract}
2D images and 3D point clouds are foundational data types for multimedia applications, including real-time video analysis, augmented reality (AR), and 3D scene understanding.  Class-incremental semantic segmentation (CSS) requires incrementally learning new semantic categories while retaining prior knowledge.  Existing methods typically rely on computationally expensive training based on stochastic gradient descent, employing complex regularization or exemplar replay. However, stochastic gradient descent-based approaches inevitably update the model’s weights for past knowledge, leading to catastrophic forgetting, a problem exacerbated by pixel/point-level granularity. To address these challenges, we propose CFSSeg, a novel exemplar-free approach that leverages a closed-form solution, offering a practical and theoretically grounded solution for continual semantic segmentation tasks. This eliminates the need for iterative gradient-based optimization and storage of past data, requiring only a single pass through new samples per step. It not only enhances computational efficiency but also provides a practical solution for dynamic, privacy-sensitive multimedia environments. Extensive experiments on 2D and 3D benchmark datasets such as Pascal VOC2012, S3DIS, and ScanNet demonstrate CFSSeg's superior performance.
\end{abstract}


\section{Introduction}

\begin{figure}[htbp]
  \centering
  \includegraphics[width=1\linewidth]{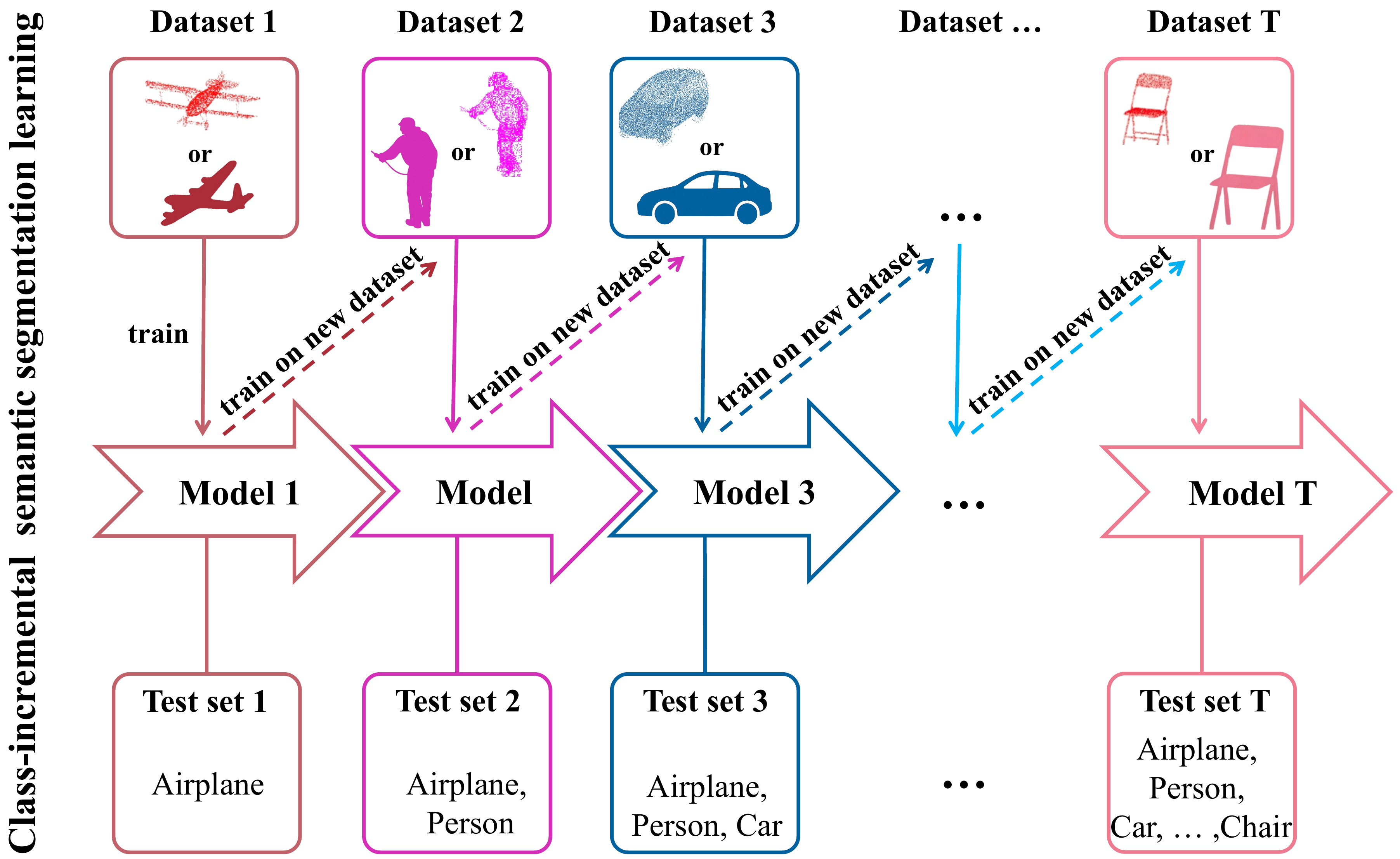}
  \caption{Illustration of the class-incremental semantic segmentation learning process. At each step, the model is incrementally trained on new classes while retaining knowledge of previously learned classes. For example, initially, the model is introduced to the "airplane" class. Subsequently, the model learns additional classes, with each step introducing new categories, progressively learning "person," "car," and "chair," among others, while expanding its knowledge and maintaining understanding of earlier classes.}
  \label{fig:1}
\end{figure}

2D images and 3D point clouds are fundamental data modalities that underpin modern multimedia applications, including real-time video analysis, augmented reality (AR), robotics, and immersive 3D scene understanding. However, real-world multimedia systems rarely use fixed and predefined sets of object categories. They often encounter new objects or concepts after initial deployment, requiring the ability to adapt and expand their knowledge base over time \cite{souza2020challenges,CIL_survey}.
A naive approach is to train models directly on newly arrived data, but this strategy is plagued by catastrophic forgetting \cite{CIL_forgetting_1}, where the model forgets previously acquired knowledge while adapting to new information. To address this issue, continual learning methods have been proposed to mitigate the effects of forgetting while allowing models to gradually adapt to new data \cite{wang2024comprehensive}.

Class-incremental semantic segmentation (CSS) \cite{ILT, MiB, douillard2021plop, SDR} presents unique challenges. As shown in Figure \ref{fig:1}, the model needs to continuously learn to segment new categories. Compared to image classification tasks, semantic segmentation tasks involve pixel/point-level granularity, which typically requires substantial computational resources and makes models more susceptible to catastrophic forgetting \cite{CSSsurvey}.  A challenge in CSS is the stability-plasticity dilemma \cite{CIL_survey}, which involves balancing two conflicting goals: stability and plasticity. Stability refers to the model’s ability to retain knowledge from past tasks, while plasticity requires the model to adapt to new incoming data. Striking the right balance is crucial for successful continual learning.

Recent advancements in CSS can be broadly categorized into exemplar-free and exemplar-based approaches. Exemplar-free methods aim to perform class-incremental learning without relying on historical data or features to reduce knowledge forgetting. These methods often employ self-supervised learning \cite{sats}, regularization techniques \cite{ILT, MiB}, or dynamic network architectures \cite{der1, der2}. On the other hand, exemplar-based methods depend on strategies such as sample replay, feature replay, auxiliary dataset integration \cite{zhu2023continual}, or pseudo-data generated by generative models \cite{yu2024tikp}. While these methods show promise in retaining knowledge, they are all based on gradient descent and inevitably erase past knowledge through gradient updates \cite{goodfellow2015empiricalinvestigationcatastrophicforgetting, GKEAL_Zhuang_CVPR2023}. Moreover, they often demand significant computational resources, and some exemplar-based methods may not be suitable in scenarios where data privacy is paramount.

Analytic learning, as an alternative to stochastic gradient descent methods, overcomes key challenges associated with backpropagation, including gradient vanishing and the instability of iterative training processes, by directly computing neural network parameters \cite{AL_1,AL_2,AL_3,AL_4,peng2025tsvd,tran2025boosting}. Inspired by this, we have proposed CFSSeg, a closed-form solution for CSS. Unlike existing incremental learning methods based on stochastic gradient descent, which require multiple training epochs, our approach needs only a single training epoch. Specifically, we freeze the encoder and update the model using a closed-form solution to achieve stability, while mapping features to a higher-dimensional space to make them more linearly separable, thereby enhancing plasticity. At the same time, it is efficient and privacy-preserving, making it suitable for practical applications. Additionally, in disjoint and overlapped settings, semantic drift can occur \cite{CSSsurvey,MiB,douillard2021plop}, where previously learned categories collapse into background class labels in new datasets. We introduce a pseudo-labeling strategy that leverages uncertainty \cite{douillard2021plop,Yang3D} to mitigate semantic drift. The overview of our method is shown in Figure \ref{fig:2}. Extensive experiments on 2D and 3D benchmark datasets such as Pascal VOC2012 \cite{Everingham10}, S3DIS \cite{s3dis}, and ScanNet \cite{scannet}  have demonstrated its superior performance.

The key contributions are summarized as follows:
\begin{itemize}
\item We propose a novel, gradient-free, closed-form solution for exemplar-free continual semantic segmentation in both 2D images and 3D point clouds.
\item We develop a recursive update mechanism for the classification head, enabling efficient single-pass incremental learning without storing past data. 
\item Through extensive experiments on Pascal VOC2012, S3DIS, and ScanNet, we demonstrate that our method achieves advanced results  while offering significant advantages in computational efficiency and data privacy.
\end{itemize}

\section{Related Work}

\noindent\textbf{Semantic Segmentation.}
Semantic segmentation, a dense prediction task, involves assigning a semantic label to every pixel in an image. In recent years, significant progress has been made in this domain, primarily driven by the development of convolutional neural network (CNN)-based models \cite{deeplabv1,deeplabv2,deeplabv3+,DeepLabv3_Chen_arXiv2017}. More recently, Transformer-based architectures \cite{SegFormer,maskformer,Mask2Former} and innovative Mamba frameworks \cite{u-mamba,Swin-umamba,RS3Mamba} have gained prominence, introducing novel methodologies and perspectives for addressing the challenges of semantic segmentation. DeeplabV3 \cite{DeepLabv3_Chen_arXiv2017} has been widely used in previous CSS work, and we selected it as our 2D segmentation model.

For the 3D point cloud modality, key methods include PointNet \cite{Qi_2017_CVPR} and its derivative architecture PointNet++ \cite{NIPS2017_d8bf84be}, which are used to directly process point cloud data; Transformer models such as Point Transformer V3 \cite{Wu_2024_CVPR}, which improve performance by capturing long-range dependencies; and DGCNN \cite{DGCNN}, which is based on the EdgeConv module, captures local neighborhood information, and learns global shape properties by stacking multiple layers. This dynamic graph approach makes it particularly suitable for handling the unstructured nature of point clouds. In this paper, we adopt DGCNN  as our 3D segmentation model due to its simplicity and effectiveness.

\noindent\textbf{2D Class-Incremental Semantic Segmentation.}
Class-incremental semantic segmentation, initially proposed in medical imaging applications \cite{medical1,medical2}, has since been extended to natural image datasets \cite{ILT,MiB}. Unlike standard classification tasks, CSS poses unique challenges due to its pixel-level granularity, which exacerbates the issue of catastrophic forgetting \cite{CSSsurvey}. CSS methods are broadly categorized into exemplar-free and exemplar-based approaches. Exemplar-free methods often leverage strategies such as self-supervised learning \cite{SDR,UCD,sats}, regularization techniques \cite{ILT,MiB,douillard2021plop}, or dynamic network architectures \cite{der1,der2,der3} to retain knowledge from previously seen data. On the other hand, exemplar-based methods employ mechanisms such as sample replay, feature replay, and auxiliary dataset integration \cite{SSUL,zhu2023continual,wang2023rethinking,kalb2022improving}, or utilize pseudo-data or pseudo-features generated by generative models \cite{maracani2021recall,yu2024tikp,liu2022new,shan2022class}, combining these with new data to enable continual training.

\noindent\textbf{3D Class-Incremental Semantic Segmentation.}
There is limited work on continual learning for 3D semantic segmentation, and it has only recently begun to be explored.  
Yang et al. \cite{Yang3D} proposed a class-incremental learning method combining geometric features and uncertainty estimation.  
LGKD \cite{LGKD} introduced a label-guided knowledge distillation loss.  
Chen et al. \cite{chen2025class} investigated class-incremental learning for mobile LiDAR point clouds, proposing strategies for feature representation preservation and loss cross-coupling.

\section{Background}
To begin, we define the objective of the semantic segmentation task. The input space is represented as $\mathcal{X} \in \mathbb{R}^{N \times C_{\text{in}}}$, where $N$ denotes the number of input elements (pixels or point clouds), and $C_{\text{in}}$ represents the number of channels per element (e.g., RGB for images, or RGB, XYZ, normals for point clouds). The output label space is $\mathcal{Y} \in \mathcal{C}^{N}$, where the set of classes is $\mathcal{C}$, including the background class $c_b \in \mathcal{C}$. Given the training dataset $\mathcal{T} = \mathcal{X} \times \mathcal{Y}$, the goal is to learn a mapping function $q$ parameterized by $\theta$ that predicts a per-element class probability distribution: $ q_\theta: \mathcal{X} \to \mathbb{R}^{N \times |\mathcal{C}|}.$ The segmentation mask is then computed as: $\hat{y} = \argmax{c \in \mathcal{C}} q_\theta (x)[i, c],$
where $q_\theta(x)[i, c]$ indicates the probability of element $i$ belonging to class $c$.

In the traditional supervised learning paradigm, the entire training set $\mathcal{T}$ is provided at once, and the model is trained in a single step. However, in continual learning, training is performed iteratively, with each step introducing new categories along with their corresponding subset of training data. This process spans multiple steps, denoted as \{step 1, step 2, $\cdots$, step $T$\}. In step $t$, the label set $\mathcal{C}_{t-1}$ is expanded by adding a new set of categories $\mathcal{S}_t$, resulting in an updated label set $\mathcal{C}_t = \mathcal{C}_{t-1} \cup \mathcal{S}_t$. Simultaneously, a new training subset $\mathcal{T}_t$ is introduced to update the previous model $q_{\theta_{t-1}}$ to $q_{\theta_{t}}$. According to the CSS principle, the newly introduced category sets are mutually exclusive, i.e., $\mathcal{S}_i \cap \mathcal{S}_j = \emptyset$ for $i \neq j$.

Different learning settings are considered for CSS, depending on the availability and labeling of categories during incremental learning. Sequential, disjoint, and overlapped settings are detailed below: \textbf{1) Sequential Setting.} In the sequential setting, labels for both previously learned and newly introduced categories are available simultaneously during each incremental learning step. \textbf{2) Disjoint Setting.} The disjoint setting introduces complexity by labeling previously learned categories as background in the current task. This phenomenon, known as semantic drift, challenges the model to differentiate between real background and previously learned classes. \textbf{3) Overlapped Setting.} The overlapped setting further complicates the learning process. Here, only new categories and the background are labeled, but the background label can encompass true background, previously learned categories, and future categories that have not yet been introduced. 

\section{Method}

\begin{figure*}[htbp]
  \centering
  \includegraphics[width=1\linewidth]{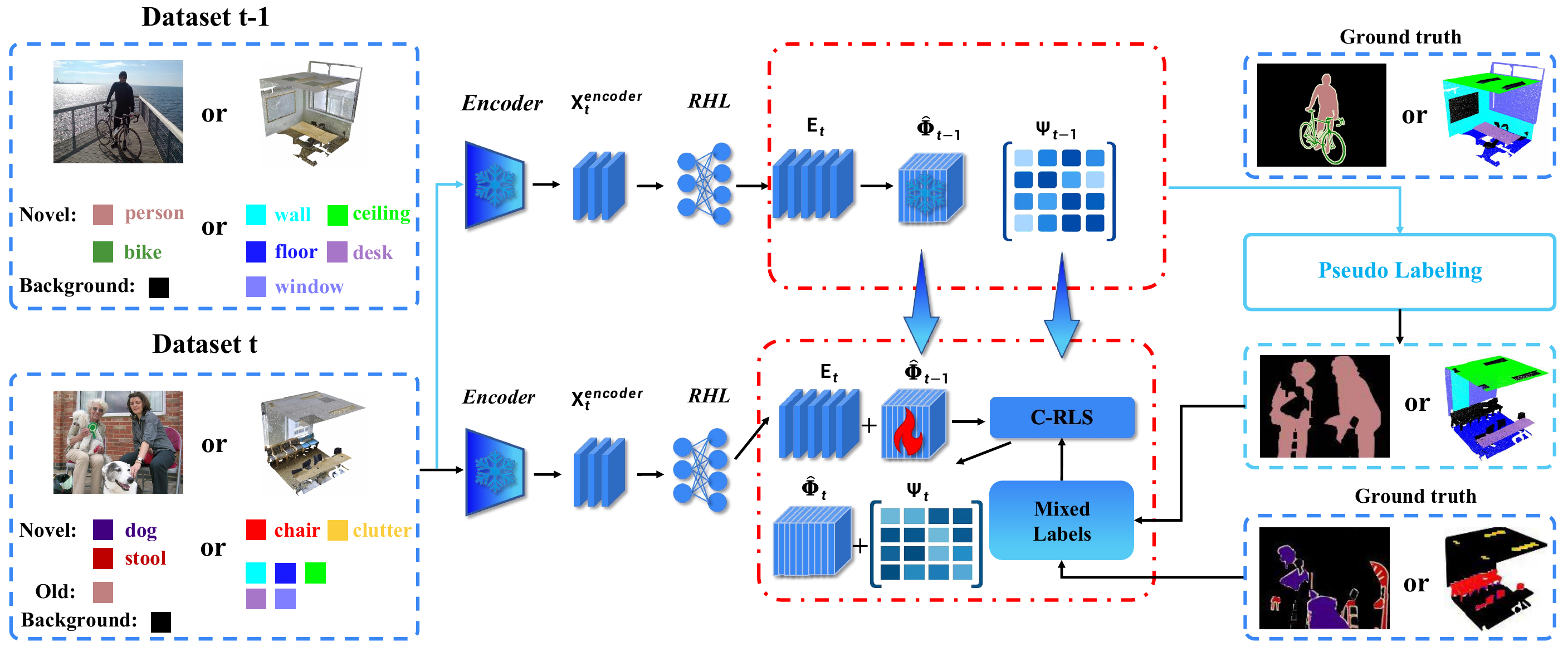}
  \caption{Overview of the proposed method CFSSeg. In step $t$, the model from step $t-1$ is used to generate pseudo labels via Pseudo Labeling, which are then combined with ground truth labels to form mixed labels. The model inherits the classification head $\hat{\mathbf{\Phi}}_{t-1}$ learned in step $t-1$, and combines it with the mixed labels, the extracted features $\mathbf{E}_t$, and $\mathbf{\Psi}_{t-1}$ from step $t-1$. The C-RLS algorithm is then used to update and obtain $\hat{\mathbf{\Phi}}_{t}$ and $\mathbf{\Psi}_{t}$.}
  \label{fig:2}
\end{figure*}

\subsection{Ridge Regression}

In step 1, we use stochastic gradient descent to train an encoder. Notably, a powerful pre-trained encoder (e.g., SAM \cite{kirillov2023segment}) can also be used to avoid this training process. We then save and freeze the encoder, treating it as a feature extractor.

After training to obtain an encoder using the gradient descent method, during the continual learning phase, we do not use the gradient descent method to train the model. This is because the update of gradients will inevitably interfere with the weights of previous tasks \cite{goodfellow2015empiricalinvestigationcatastrophicforgetting,GKEAL_Zhuang_CVPR2023}, leading to forgetting. Therefore, we adopt a simpler ridge regression, which has a closed-form solution. 

Although freezing the backbone resolves the stability issue, it affects the model's plasticity. In order to increase the plasticity, according to Cover's Theorem \cite{RHL}, non-linearly mapping the features to a high-dimensional space can increase the probability of the features being linearly separable, we adopt a simple high-dimensional mapping method. The features extracted from the encoder are passed through a Randomly-initialized Hidden Layer (RHL) followed by a non-linear activation function (ReLU). The RHL is a linear layer whose weights are initialized from a normal distribution. Let $\mathbf{X}_{1}^{\text{encoder}}$ and $\mathbf{Y}_{1}^{\text{train}}$ be the  feature matrix extracted by the encoder and the label matrix in step 1. We denote 
\begin{equation}
\mathbf{E}_1 = \text{ReLU}(\mathbf{X}_1^{\text{encoder}} \mathbf{\Phi}^{\mathrm{E}}),
\label{extrat features}
\end{equation} where $\mathbf{\Phi}^{\mathrm{E}}  \in \mathbb{R}^{d_{\text{encoder}} \times d_{\mathrm{E}}}$, and $d_{\text{encoder}}$ is the number of channels of the features extracted by the encoder. Generally, $d_{\text{encoder}} \ll d_{\mathrm{E}}$.  The features $\mathbf{E}_1 \in \mathbb{R}^{K_1 \times d_\mathrm{E}}$ are employed to predict the label matrix $\mathbf{Y}_1^{\text{train}} \in \mathbb{R}^{K_1 \times C_1}$, where $K_1$ is the input number of elements. The matrix $\mathbf{Y}_1^{\text{train}}$ is a label matrix formed by stacking one-hot labels. This prediction is carried out via the ridge regression approach, which entails solving the following optimization problem:

\begin{equation}
\argmin{\mathbf{\Phi}_{1}} (\| \mathbf{Y}_1^{\text{train}} - \mathbf{E}_1 \mathbf{\Phi}_{1} \|_\text{F}^2 + \gamma \| \mathbf{\Phi}_{1} \|_\text{F}^2),
\end{equation}

\noindent where $\gamma$ is a regularization parameter, $\|\cdot\|_{\text{F}}$ represents the Frobenius norm. The optimal solution to this problem is:
\begin{equation}
    \widehat{\mathbf{\Phi}}_{1} = ( \mathbf{E}_1^{\top} \mathbf{E}_1 + \gamma \mathbf{I} )^{-1} \mathbf{E}_1^{\top} \mathbf{V}_1^{\text{train}},
    \label{eq:solution}
\end{equation}
where $\mathbf{I}$ is the identity matrix, $^{-1}$ represents the matrix inversion operation, and $^\top$ denotes the matrix transpose operation.

\subsection{Recursive Ridge Regression for CSS}
The previous subsection introduced ridge regression learning, which, however, is not suitable for continual learning. Next, we will propose the concatenated recursive least squares (C-RLS) algorithm, the pseudocode of which is shown in Algorithm \ref{alg:cfssg_incremental}. Without loss of generality, let $\mathbf{Y}_{1:t-1}^{\text{train}}$, $\mathbf{Y}_{1:t}^{\text{train}}$ and $\mathbf{E}_{1:t-1}$, $\mathbf{E}_{1:t}$, be the accumulated label and feature matrices in step $t-1$ and $t$, and they are related via
\begin{equation}
	\mathbf{Y}_{1:t}^{\text{train}} = \begin{bmatrix}
		\mathbf Y_{1:t-1}^{\text{train}} & \mathbf{0}\\
		\mathbf {\bar Y}_{t}^{\text{train}} & \mathbf {\tilde{Y}}_{t}^{\text{train}} 
	\end{bmatrix}, \quad
	\mathbf{E}_{1:t} = \begin{bmatrix}
		\mathbf{E}_{1:t-1}\\
		\mathbf{E}_{t}
	\end{bmatrix}.
\end{equation}

The block matrix is due to the covered-uncovered partition
\begin{equation}
	\mathbf{Y}_{t}^{\text{train}} = \begin{bmatrix}
		\mathbf {\bar Y}_{t}^{\text{train}} & \mathbf {\tilde{Y} } _{t}^{\text{train}} 
	\end{bmatrix},
\end{equation}

\noindent where $\mathbf{\bar Y}_{t}^{\text{train}}\in\mathbb{R}^{N_t \times d_{C_{t-1}}}$ is the \textit{projected covered matrix} and ${\mathbf{\tilde{Y}}_{t}}^{\text{train}}\in\mathbb{R}^{N_t \times (d_{C_{t}}-d_{C_{t-1}})}$ is the \textit{projected uncovered matrix}. They correspond to segments displaying the appearance of covered classes and uncovered classes, representing the already learned classes and the yet-to-be-learned classes, respectively. The learning problem can then be formulated as:
\begin{equation}
\argmin{\mathbf{\Phi}_{t-1}} (\| \mathbf{Y}_{1:t-1}^{\text{train}} - \mathbf{E}_{1:t-1} \mathbf{\Phi}_{t-1} \|_\text{F}^2 + \gamma \| \mathbf{\Phi}_{t-1} \|_\text{F}^2),
\end{equation}

\noindent according to Eqn \eqref{eq:solution}, at step $ t-1 $, we have:
\begin{equation}
\hat{\mathbf{\Phi}}_{t-1} = ( \mathbf{E}^{\top}_{1:t-1} \mathbf{E}_{1:t-1} + \gamma \mathbf{I} )^{-1} \mathbf{E}^{\top}_{1:t-1} \mathbf{Y}_{1:t-1}^{\text{train}},
\label{eq:9}
\end{equation}

\noindent where $ \hat{\mathbf{\Phi}}_{t-1} \in \mathbb{R}^{d_{\mathrm{E}} \times \sum_{i=1}^{t-1} d_{S_i}} $, with the column size expanding as $ t $ increases. Let
\begin{equation}
\mathbf{\Psi}_{t-1} = ( \mathbf{E}^{\top}_{1:t-1} \mathbf{E}_{1:t-1} + \gamma \mathbf{I} )^{-1}
\label{eq:10}
\end{equation}
\noindent be the inverted auto-correlation matrix, which captures the correlation information from both current and past samples. Building on this, the goal is to compute $ \hat{\mathbf{\Phi}}_{t} $ using only $ \hat{\mathbf{\Phi}}_{t-1}$, $ \mathbf{\Psi}_{t-1} $, and the current step's data, without involving historical samples. The process is formulated as shown in the following theorem.

\begin{theorem}
\label{theorem:main}
The $\mathbf{\Phi_t}$ weights, recursively obtained by
\begin{equation}
\hat{\mathbf{\Phi}}_{t} = \left[ \hat{\mathbf{\Phi}}_{t-1} - \mathbf{\Psi}_{t} \mathbf{E}^{\top}_t \mathbf{E}_t \hat{\mathbf{\Phi}}_{t-1}+\mathbf{\Psi}_{t} \mathbf{E}^{\top}_t \mathbf{\bar{Y}}^{\text{train}}_t \quad \mathbf{\Psi}_t \mathbf{E}^{\top}_t \mathbf{\tilde{Y}}^{\text{train}}_t \right]
\label{eq:phi}
\end{equation}
\noindent are equivalent to those obtained from Eqn \eqref{eq:9} for step $ t $. The matrix $ \mathbf{\Psi}_t $ can also be recursively updated by
\begin{equation}
    \mathbf{\Psi}_t = ( \mathbf{\Psi}_{t-1}^{-1} + \mathbf{E}^{\top}_t \mathbf{E}_t )^{-1}.
    \label{eq:psi}
\end{equation}
\end{theorem}
\begin{proof}
See the supplementary materials.
\end{proof}


\subsection{Theoretical Analysis}
    \noindent\textbf{Privacy Protection.}  Our method ensures data privacy in two ways: first, by eliminating the need to store historical data samples; second, by guaranteeing that historical raw data samples  cannot be recovered from the $\mathbf{\Psi}$ matrix through reverse engineering.

    \noindent\textbf{Computational Complexity.} The computational complexity analysis reveals that the time complexity for each step includes $\mathcal{O}(d_{\mathrm{E}}^3)$ for updating $\mathbf{\Psi}_t$ via matrix inversion, and $\mathcal{O}(d_{\mathrm{E}}^2N_t + d_{\mathrm{E}}N_t^2 + d_{\mathrm{E}}^2C_t)$ for updating $\mathbf{\Phi}_t$ via matrix multiplication. These operations can be efficiently parallelized on GPU. 

    \noindent\textbf{Space Complexity.}The space complexity is $\mathcal{O}(d_{\mathrm{E}}^2 + d_{\mathrm{E}}N_t + d_{\mathrm{E}}C_t)$: $\mathcal{O}(d_{\mathrm{E}}^2)$ is for storing the $\mathbf{\Psi}_t$ matrix, $\mathcal{O}(d_{\mathrm{E}}N_t)$ for storing the feature matrix $\mathbf{E}_t$, $\mathcal{O}(d_{\mathrm{E}}C_t)$ for storing the classifier matrix $\mathbf{\Phi}_t$.

    \label{efficiency}

\subsection{Pseudo-Labeling for 2D Image}
At step $t$, $(x_t, y_t) \in \mathcal{T}_t$, where $x_t^i, y_t^i$ represent the elements and their corresponding ground truth labels, respectively. In both disjoint and overlapped settings, previously learned classes are treated as background in the current task, a phenomenon commonly referred to as semantic drift. To address this issue, we adopt a pseudo-labeling approach. We define the uncertainty of an element as follows:  

\begin{equation}  
U_i = 1-\sigma(\max\limits_{c} (q_{\theta_{t-1}}(i,c)),  
\end{equation}  
where $q_{\theta_{t-1}}(i,c)$ denotes the logit output of the model for element $i$ in class $c$, $\sigma$ represents the Sigmoid activation function, and $\hat{y}^{t-1}_i=\argmax{c}{q_{\theta_{t-1}}(i,c)}$ denotes the predicted label from the model at step $t-1$.  The pseudo-labeling strategy is then defined as:  

\begin{equation}  
\tilde{y}^i_t =   
\begin{cases}   
y^i_t, & y^i_t \in S_t, \\  
{y}_{t}^i, & ( y^i_t = c_b ) \wedge ( U_i > \tau ), \\  
\hat{y}_{t-1}^i, & (y^i_t = c_b) \wedge ( U_i \leq \tau ),\\  
\end{cases}  
\label{2D pseudo-label}
\end{equation}  
where $\tau$ is a threshold that determines whether the pixel labeled as background  should adopt pseudo-label generated by the prior model. This approach mitigates the semantic drift issue in both disjoint and overlapped settings.

\subsection{Pseudo-Labeling for 3D Point Cloud}

For point cloud $ i $, we employ the KNN algorithm to identify its $ K $ nearest neighbors based on the $ xyz $ coordinates  and compute the cosine similarity $ w_k $ between point cloud $ i $ and its $ K $ neighbors using their $ xyz $ coordinates.  We adopt a neighborhood spatial aggregation method based on Monte Carlo dropout (MC-dropout) technique, which achieves efficient estimation of point distribution uncertainty through a single forward propagation. This approach utilizes a spatial dependency sampling mechanism, and its effectiveness has been validated in the literature \cite{9560972,Yang3D}. For uncertainty quantification, we employ the Bayesian Active Learning Disagreement (BALD) criterion \cite{BALD} as the core evaluation function for point cloud spatial sampling. Specifically, given an input point cloud, its uncertainty metric function can be expressed as:  

\begin{equation}
\begin{aligned}
U_i &= - \sum\limits_c \left[ \frac{1}{K} \sum\limits_k q_{t-1}(i, c) \cdot w_k \right] \log \left[ \frac{1}{K} \sum\limits_k q_{t-1}(i, c) \cdot w_k \right] \\
&\quad + \frac{1}{K} \sum\limits_{c,k} \left( q_{t-1}(i, c) \cdot w_k \right) \log \left( q_{t-1}(i, c) \cdot w_k \right).
\end{aligned}
\end{equation}  
And then the pseudo-labeling method is defined as:  

\begin{equation}
\tilde{y}^i_t = 
\begin{cases} 
y_t^i, & y_t^i \in S_t, \\ 
\hat{y}^i_{t-1}, & (y_t^i = c_b) \land (\hat{y}^i_{t-1} \neq c_b) \land (U_i \leq \tau), \\ 
\hat{y}^{i,k'}_{t-1}, & (y^i_t = c_b) \land \left( (\hat{y}^i_{t-1} = c_b) \lor (U_i > \tau) \right), \\ 
c_b, & \text{otherwise},
\end{cases}
\label{3D pseudo-label}
\end{equation}  
where $\hat{y}^{i,k'}_{t-1}$ represents the predicted label of the nearest neighbor of point cloud $i$, with $\hat{y}^{i,k'}_{t-1} \neq c_b$ and $U_{i,k'} \leq \tau$. It is worth noting that we may not be able to find such a point, in which case we mark the point as belonging to the background class.

\begin{table*}[htbp]
    \caption{CSS Quantitative comparison of 3D class-incremental segmentation methods on the S3DIS and ScanNet datasets in $S^{0}$ split. BT stands for training only on the dataset from step 1. The best results achieved by the incremental methods are highlighted in bold.\\}
    \label{tab:experiment:s0}
    \centering
    \setlength{\tabcolsep}{5.5pt} 
    \renewcommand{\arraystretch}{1.05} 
    \scalebox{1}{ 
    \begin{tabular}{ll| c| c| ccc| ccc| ccc}
             \toprule[1pt]
            &\multirow{2}*{Method} &\multirow{2}*{Year}&\multirow{2}*{Model}&
		\multicolumn{3}{c|}{8-1 (6 steps)} & \multicolumn{3}{c|}{10-1 (2 steps)} &         
            \multicolumn{3}{c}{12-1 (1 steps)} \\
		&&&& 0-7 & 8-12 & all & 0-9 & 10-12 & all & 0-11 & 12-1 & all\\

            \midrule 
            \multirow{6}*{S3DIS}
            & BT   & -    & DGCNN     & 49.85 &   -   &   -   & 47.03 &   -   &   -   & 45.37 &   -   &   -   \\
            & FT   & -    & DGCNN     & 5.45  & 4.94  & 5.25  & 16.66 & 11.54 & 15.48 & 28.34 & 34.02 & 28.77 \\
            & EWC \cite{EWC} & PNAS2017 & DGCNN & 30.99 & 4.16 & 20.67 & 39.14 & 14.73 & 33.51 & 37.23 & 15.49 & 35.53 \\
            & Yang et al. \cite{Yang3D} & CVPR2023 & DGCNN & 36.90 & 13.76 & 27.80 & 36.98 & 26.60 & 34.59 & 44.54 & \textbf{38.47} & 44.07 \\
            & Ours &    & DGCNN & \textbf{49.77} & \textbf{28.69} & \textbf{41.66} & \textbf{45.26} & \textbf{34.05} & \textbf{42.67} & \textbf{45.19} & 30.71 & \textbf{44.08} \\
            
            & JT   & -    & DGCNN     & 48.99 & 41.76 & 46.21 & 46.72 & 45.11 & 46.35 & 47.22 & 35.65 & 46.33 \\
            \midrule[1pt] 
            &\multirow{2}*{Method} &\multirow{2}*{Year}&\multirow{2}*{Model}&
		\multicolumn{3}{c|}{15-1 (6 steps)} & \multicolumn{3}{c|}{17-1 (2 steps)} &         
            \multicolumn{3}{c}{19-1 (1 steps)} \\
		&&&& 0-14 & 15-19 & all & 0-16 & 17-19 & all & 0-18 & 19 & all\\

            \midrule 
            \multirow{6}*{ScanNet}
            & BT   & -    & DGCNN     & 37.52 &   -   &   -   & 34.72 &   -   &   -   & 32.64 &   -   &   -   \\
            & FT   & -    & DGCNN     & 4.41  & 3.20  & 4.10  & 2.76  & 4.19  & 2.97  & 11.37  & 8.86 & 11.25  \\
            & EWC \cite{EWC} & PNAS2017 & DGCNN & 12.32 & 2.48 & 9.86 & 11.34 & 1.20 & 9.82 & 17.84 & 11.11  & 17.51 \\
            & Yang et al. \cite{Yang3D} & CVPR2023 & DGCNN & 8.46 & 4.44 & 7.46  & 12.29 & 6.48 & 11.42 & 25.56 & 11.31 & 24.85 \\
            & Ours &     & DGCNN & \textbf{32.56} & \textbf{10.22} & \textbf{26.97} & \textbf{29.59} & \textbf{12.83} & \textbf{27.98} & \textbf{28.54} & \textbf{16.88} & \textbf{27.96} \\
            
            & JT   & -    & DGCNN     & 37.59 & 16.42 & 32.30 & 34.78 & 16.69 & 32.07 & 32.28 & 18.12 & 32.08 \\
            \bottomrule[1pt]
    \end{tabular}
    } 
\end{table*}

\begin{table*}[htbp]
    \centering
    \caption{CSS Quantitative comparison of 3D class-incremental segmentation methods on the S3DIS and ScanNet datasets in $S^{1}$ split. BT stands for training only on the dataset from step 1. The best results achieved by the incremental methods are highlighted in bold.\\}
    \label{tab:experiment:s1}
    \setlength{\tabcolsep}{5.5pt} 
    \renewcommand{\arraystretch}{1.05} 
    \scalebox{1}{ 
    \begin{tabular}{ll| c| c| ccc| ccc| ccc}
             \toprule[1pt]
            &\multirow{2}*{Method} &\multirow{2}*{Year}&\multirow{2}*{Model}&
		\multicolumn{3}{c|}{8-1 (6 steps)} & \multicolumn{3}{c|}{10-1 (2 steps)} &         
            \multicolumn{3}{c}{12-1 (1 steps)} \\
		&&&& 0-7 & 8-12 & all & 0-9 & 10-12 & all & 0-11 & 12-1 & all\\
            \midrule 
            \multirow{6}*{S3DIS}
            & BT   & -    & DGCNN     & 37.61 &   -   &   -   & 42.93 &   -   &   -   & 44.68 &   -   &   -   \\
            & FT   & -    & DGCNN     & 12.69 &  2.57 & 8.80  & 10.11 &  2.33 & 8.31  & 25.05 &  11.75 & 24.04 \\
            & EWC \cite{EWC} & PNAS2017 & DGCNN & 17.50 & 2.59 & 11.77 & 20.54 & 8.86 & 17.84 & 23.20 & 11.40 & 22.38 \\
            & Yang et al. \cite{Yang3D} & CVPR2023 & DGCNN & 41.28 &11.64 & 29.88  & 36.45 & 20.14 & 31.13 & 37.38 & 22.40 & 36.23 \\
            & Ours &     & DGCNN     & \textbf{51.33} & \textbf{30.72} & \textbf{43.40} & \textbf{45.16} & \textbf{35.98} & \textbf{43.04} & \textbf{45.65} & \textbf{27.53} & \textbf{44.25} \\
            
            & JT   & -    & DGCNN     & 39.29 & 57.42 & 46.26 & 42.19 & 59.08 & 46.09 & 46.96 & 38.89 & 46.35 \\
            \midrule[1pt] 
            &\multirow{2}*{Method} &\multirow{2}*{Year}&\multirow{2}*{Model}&
		\multicolumn{3}{c|}{15-1 (6 steps)} & \multicolumn{3}{c|}{17-1 (2 steps)} &         
            \multicolumn{3}{c}{19-1 (1 steps)} \\
		&&&& 0-14 & 15-19 & all & 0-16 & 17-19 & all & 0-18 & 19 & all\\
            \midrule 
            \multirow{6}*{ScanNet}
            & BT   & -    & DGCNN     & 29.01 &   -   &   -   & 30.24 &   -   &   -   & 31.59 &   -   &   -   \\
            & FT   & -    & DGCNN     & 4.20  & 1.61  & 3.55  &  3.63 & 1.00  & 3.24  &  9.04 &  0.22 &  8.60 \\
            & EWC \cite{EWC} & PNAS2017 & DGCNN & 14.93 & \textbf{33.30} & 19.52 &  8.78 & \textbf{31.74} & 12.22 & 12.65 &  3.19 & 12.18 \\
            & Yang et al. \cite{Yang3D} & CVPR2023 & DGCNN & 12.17 & 0.16 & 9.17 & 10.19 & 5.67 & 9.52 & 25.08 & \textbf{15.28} & 24.59 \\
            & Ours &     & DGCNN     & \textbf{28.20} & 16.09 & \textbf{25.18} & \textbf{28.43} & 15.01 & \textbf{26.42} & \textbf{28.71} & 11.84 & \textbf{27.84} \\
            
            & JT   & -    & DGCNN     & 28.93 & 39.06 & 31.46 & 30.12 & 40.33 & 31.65 & 31.29 & 30.03 & 31.23 \\
            \bottomrule[1pt]
    \end{tabular}
    } 
\end{table*}

\begin{table*}[htbp]
	\caption{CSS quantitative comparison on Pascal VOC2012 in mIoU (\%) under \textit{sequential} setting. The results of the comparison method are directly taken from the original work and \cite{CSSsurvey}. For the best results, we use \textbf{bold} formatting.}
       
	\centering
	\scriptsize
	\setlength{\tabcolsep}{4.8mm}{
		{\begin{tabular*}{1.0\textwidth}{ll|c|c|ccc|ccc}
				\toprule[0.5mm]
				&\multirow{2}*{Method} &\multirow{2}*{Year}&\multirow{2}*{Model}&
				\multicolumn{3}{c|}{15-1 (6 steps)} & \multicolumn{3}{c}{15-5 (2 steps)} \\
				&&&& 0-15 & 16-20 & all & 0-15 & 16-20 & all\\
				\midrule

				\multirow{9}*{\begin{sideways}Sequential\end{sideways}} &\emph{FT}&-&DeepLabv3+&49.0&17.8&41.6&62.0&38.1&56.3\\
				&LwF~\cite{LWF} &TPAMI2018&DeepLabv3+&33.7 &13.7 &29.0&68.0 &43.0 &62.1\\
				&LwF-MC~\cite{iCaRL1}&CVPR2017&DeepLabv3+&12.1&1.9&9.7&70.6&19.5&58.4\\
				&ILT~\cite{ILT} &ICCVW2019&DeepLabv3+&49.2&30.3&48.3&71.3&\textbf{47.8}&65.7\\
    			&CIL~\cite{CIL} &ITSC2020&DeepLabv3+&52.4&22.3&45.2&63.8&39.8&58.1\\
                    &MiB~\cite{MiB}&CVPR2020&DeepLabv3+&35.7&11.0&29.8&73.0&44.4&66.1\\
    
                    &SDR~\cite{SDR}&CVPR2021&DeepLabv3+&58.5&10.1&47.0&73.6&46.7&67.2\\

                    &SDR+MiB~\cite{SDR}&CVPR2021&DeepLabv3+&58.1&11.8&47.1&74.6&43.8&67.3\\
                    & Ours & & DeepLabv3 & \textbf{78.1} & \textbf{42.0} & \textbf{70.0} & \textbf{78.1} & 42.0 & \textbf{70.0} \\
        \bottomrule[0.5mm]
		\end{tabular*}}{}}	
	\label{table-VOC2012:1}
\end{table*}

\begin{table*}[htbp]
	\caption{CSS quantitative comparison on Pascal VOC2012 in mIoU (\%) under \textit{disjoint} and \textit{overlapped} settings. The results of the comparison method are directly taken from the original work and \cite{CSSsurvey}. For the best results, we use \textbf{bold} formatting.}

	\centering
	\scriptsize
	\setlength{\tabcolsep}{4.8mm}
		{\begin{tabular*}{1\textwidth} {ll|c|c|ccc|ccc}
				\toprule[0.5mm]
				&\multirow{2}*{Method} &\multirow{2}*{Year}&\multirow{2}*{Model}&
				\multicolumn{3}{c|}{15-1 (6 steps)} & \multicolumn{3}{c}{10-1 (11 steps)} \\
				&&&& 0-15 & 16-20 & all & 0-10 & 11-20 & all\\
				\midrule
                
				\multirow{6}*{\begin{sideways}Disjoint\end{sideways}} &\emph{FT}&-&DeepLabv3&0.20&1.80&0.60&6.30&1.10&3.80\\
				&MiB~\cite{MiB}&CVPR2020&DeepLabv3 &46.20&12.90&37.90&9.50&4.10&6.90 \\
				&PLOP~\cite{douillard2021plop}&CVPR2021&DeepLabv3&57.86&13.67&46.48&9.70&7.00&8.40 \\
				&SDR~\cite{SDR} &CVPR2021&DeepLabv3+ &59.40&14.30&48.70&17.30&11.00&14.30\\ 
				&RCIL~\cite{RCIL}&CVPR2022&DeepLabv3&66.10&18.20&54.70&30.60&4.70&18.20 \\
                    &Ours& &DeepLabv3&\textbf{77.66}&\textbf{40.33}&\textbf{68.77}&\textbf{70.85}&\textbf{42.13}&\textbf{57.17} \\
				\midrule
				\multirow{15}*{\begin{sideways}Overlapped\end{sideways}} &\emph{FT}&-&DeepLabv3&0.20&1.80&0.60&6.30&2.80&4.70\\
				&EWC~\cite{EWC} &PNAS2017&DeepLabv3&0.30 &4.30 &1.30&- &- &-\\
				&LwF-MC~\cite{iCaRL1}&CVPR2017&DeepLabv3 &6.40&8.40&6.90&4.65&5.90&4.95\\
				&ILT~\cite{ILT} &ICCVW2019&DeepLabv3&4.90&7.80&5.70&7.15&3.67&5.50\\
				&MiB~\cite{MiB}&CVPR2020&DeepLabv3&34.22&13.50&29.29&12.25&13.09&12.65\\
				&PLOP~\cite{douillard2021plop}&CVPR2021&DeepLabv3&65.12&21.11&54.64&44.03&15.51&30.45 \\
				&UCD+PLOP~\cite{UCD}&TPAMI2022&DeepLabv3&66.30&21.60&55.10&42.30&28.30&35.30\\
				&REMINDER~\cite{REMINDER} &CVPR2022&DeepLabv3&68.30&27.23&58.52&-&-&- \\
				&RCIL~\cite{RCIL}&CVPR2022&DeepLabv3&70.60&23.70&59.40&55.40&15.10&34.30\\
				&SPPA~\cite{SPPA}&ECCV2022&DeepLabv3&66.20&23.30&56.00&-&-&-\\
				&CAF~\cite{CAF}&TMM2022&DeepLabv3&55.70&14.10&45.30&-&-&-\\

				&AWT+MiB~\cite{AWT}&WACV2023&DeepLabv3 &59.10&17.20&49.10&33.20&18.00&26.00\\
				&EWF+MiB~\cite{EWF}&CVPR2023&DeepLabv3&78.00&25.50&65.50&56.00&16.70&37.30\\

				&GSC~\cite{GSC}&TMM2024&DeepLabv3&72.10&24.40&60.80&50.60&17.30&34.70 \\
                    &Ours& &DeepLabv3&\textbf{79.16}&\textbf{38.00}&\textbf{69.36}&\textbf{75.02}&\textbf{41.20}&\textbf{58.91} \\
				\midrule
				
				&\emph{JT} &-&DeepLabv3&79.77&72.35&77.43&78.41&76.35&77.43\\

				\bottomrule[0.5mm]
		\end{tabular*}}
  
	\label{table-VOC2012:2}
\end{table*}

\begin{algorithm}[htbp] 
\caption{CFSSeg: Closed-Form Solution for CSS}
\label{alg:cfssg_incremental}
\begin{algorithmic}[1]
\Require  Sequence of subsequent training datasets $\{\mathcal{T}_2, \dots, \mathcal{T}_T\}$, Initial classifier $\mathbf{\hat\Phi_1}$, Initial inverse auto-correlation $\mathbf{\Psi_1}$
\Ensure Final Classifier $\mathbf{\hat\Phi_T}$, Final inverse auto-correlation $\mathbf{\Psi_T}$
\For{$t = 2$ to $T$}
    \State Extract features $\mathbf{E_t}$ using Eqn (\ref{extrat features})
    \State Generate mixed labels using Eqn (\ref{2D pseudo-label}) or Eqn (\ref{3D pseudo-label})
    \State Use C-RLS algorithm:
    \State \quad  Update $ \mathbf{\Psi_{t}}$ using Eqn (\ref{eq:psi}) 
    \State \quad  Update $ \mathbf{\hat\Phi_{t}}$ using Eqn (\ref{eq:phi})
    \State Set $\mathbf{\Psi_{t-1}} \leftarrow \mathbf{\Psi_t}$
    \State Set $\mathbf{\hat\Phi_{t-1}} \leftarrow \mathbf{\hat\Phi_t}$ 
\EndFor

\State \Return $\mathbf{\hat\Phi_T}$, $\mathbf{\Psi_T}$ 

\end{algorithmic}
\end{algorithm}

\section{Experiments}
\subsection{Experimental Setup}

\noindent \textbf{2D Dataset.}
We evaluate our method using public 2D semantic segmentation benchmarks: Pascal VOC2012 \cite{Everingham10}. It contains 21 classes (including background class). This dataset features wild scenes, with 10,582 images used for training and 1,449 images for validation.

\noindent \textbf{3D Dataset.}
We evaluate our method using two  public 3D point cloud segmentation benchmarks: S3DIS \cite{s3dis} and ScanNet \cite{scannet}. These datasets are selected for their diversity, relevance to our problem domain, and ability to facilitate fair comparisons with existing benchmark methods \cite{EWC,Yang3D}. S3DIS comprises point clouds from 272 rooms across 6 indoor areas, with each point containing xyz coordinates and RGB information, manually annotated with one of 13 predefined classes. Following standard practice \cite{Yang3D}, we designate the more challenging Area 5 as the validation set, while the remaining areas are used for training. ScanNet, on the other hand, is an RGB-D video dataset featuring 1,513 scans from 707 indoor scenes. Each point is labeled with one of 21 classes, including 20 semantic classes and an additional category for unannotated places. Adhering to the standard dataset splits \cite{scannet}, we allocate 1,210 scans for training and 312 scans for validation. We adopt a sliding window \cite{qi2017pointnet++,wang2019dynamic,Yang3D} to partition the rooms in the S3DIS and ScanNet datasets, generating 7,547 and 36,350 1m×1m blocks respectively, and randomly sample 2,048 points from each block as input data. We use two sequences, $S^0$ and $S^1$, to partition the 3D dataset. $S^0$ follows the original dataset's annotation order, while $S^1$ follows the alphabetical order of the category names.

\noindent \textbf{CSS Learning Protocol.}
The classes of the images for the current step include $\mathcal{C}_{t-1} \cup \mathcal{S}_t$. In each step, we continuously introduce new classes for learning.   In an $m-n$ setting, the model first learns $m$ classes, and in each subsequent step, it incrementally learns $n$ classes. For 2D CSS, we adopt the three settings: sequential, disjoint, and overlapped. For 3D CSS, we follow \cite{Yang3D} and use the disjoint setting.

\noindent \textbf{Evaluation Metrics.} We use the widely adopted mean Intersection-over-Union (mIoU) metric to calculate the average IoU value across all classes. The IoU for a single class is computed as: $\text{IoU} = \frac{TP}{TP + FP + FN}$, where $TP$, $FP$, and $FN$ represent true positives, false positives, and false negatives respectively. To comprehensively evaluate CSS performance, we compute mIoU values separately for initial classes $ \mathcal{C}_{1} $, incremental classes $ \mathcal{C}_{T} \setminus \mathcal{C}_{1} $, and all classes $ \mathcal{C}_{T} $.

\noindent  \textbf{Comparison Methods.} Our method is an exemplar-free  method. For fairness, we also compare it with other exemplar-free methods. For 2D CSS, see Table \ref{table-VOC2012:1} and Table \ref{table-VOC2012:2} for details. As for 3D CSS, due to the limited number of relevant baselines, we consider Yang et al.'s method \cite{Yang3D} as a strong baseline because it is the only open-source SOTA method to the best of our knowledge, along with the classic baseline EWC \cite{EWC}. At the same time, we establish a naive baseline: FT, which fine-tunes both the backbone and the classification head. In addition, we include an upper bound, namely JT, which stands for joint training.

\noindent\textbf{Implementation Details.}  
For the initial  training in step 1, we adopt DeepLabv3~\cite{DeepLabv3_Chen_arXiv2017} with a ResNet-101 \cite{resnet} backbone pre-trained on ImageNet-1K.  We set the number of epochs to 50 and the batch size to 32. We use SGD as the optimizer with a learning rate of \(10^{-2}\), a momentum of 0.9, and a weight decay of \(10^{-4}\), combined with a polynomial learning rate scheduler. The loss function is binary cross-entropy (BCE). For the 3D CSS encoder, we employ DGCNN~\cite{DGCNN} with a batch size of 32 and the Adam optimizer, using an initial learning rate of 0.001 and a weight decay of 0.0001 for 100 epochs. In the continual learning step,
we freeze the encoder and insert an RHL layer. In the 2D experiments, we set \(d_\mathrm{E}\) to 8192, \(\gamma\) to 1, and $\tau$ to 0.4. In the 3D experiments, we set \(d_\mathrm{E}\) to 5000, \(\gamma\) to 1, and $\tau$ to 0.0035 and 0.001 on the S3DIS and ScanNet datasets via cross-validation respectively.

\subsection{Main Results}  

\noindent\textbf{2D Experimental Results.}  
Extensive experiments on the Pascal VOC2012 dataset demonstrate the outstanding performance of our method across all evaluation settings (Tables \ref{table-VOC2012:1} and \ref{table-VOC2012:2}). Under the sequential 15-1 configuration, our approach achieves an overall mIoU of 70.0\%, significantly surpassing current state-of-the-art methods. The method maintains strong performance on base classes (78.1\% mIoU) while retaining high accuracy for novel classes (42.0\% mIoU), effectively addressing the inherent stability-plasticity dilemma in continual learning. Notably, identical performance metrics in both the 15-1 and 15-5 sequential settings confirm the mathematical consistency of our closed-form solution. The method's advantages become even more pronounced in challenging scenarios: under the disjoint 15-1 setting, it achieves 68.77\% mIoU across all classes (77.66\% for base classes, 40.33\% for novel classes), while in the disjoint 10-1 setting, the performance gap widens dramatically (57.17\% vs. 18.20\% for the next best competitor). Similarly, in overlapped scenarios, our method maintains its superiority with 69.36\% mIoU in the 15-1 configuration and 58.91\% in the 10-1 setting, outperforming existing approaches across all class categories. These compelling results validate the theoretical advantages of the closed-form solution in mitigating catastrophic forgetting while efficiently integrating new knowledge.  

\noindent\textbf{3D Experimental Results.}  
Our method demonstrates exceptional performance in 3D point cloud segmentation tasks across multiple benchmark datasets. On the S3DIS dataset with the $S^0$ split, it outperforms existing approaches in all evaluation protocols: the 8-1 configuration achieves an overall mIoU of 41.66\% (49.77\% for base classes, 28.69\% for novel classes), while the 10-1 and 12-1 configurations reach 42.67\% and 44.08\% mIoU, respectively. With the $S^1$ split, the 8-1 configuration improves further to 43.40\% mIoU (51.33\% for base classes, 30.72\% for novel classes), setting a new state-of-the-art benchmark for 3D continual semantic segmentation. The method also proves effective on the more challenging ScanNet dataset: under the $S^0$ split, the 15-1 configuration achieves 26.97\% mIoU, significantly outperforming previous methods that struggled to exceed 10\% mIoU. The 17-1 and 19-1 configurations yield overall mIoU scores of 27.98\% and 27.96\%, respectively. On the $S^1$ split, the three configurations achieve 25.18\%, 26.42\%, and 27.84\% mIoU. The consistent performance across diverse 3D configurations underscores the versatility and robustness of our closed-form solution in the point cloud domain, providing a theoretically grounded and practical solution for real-world 2D and 3D semantic segmentation applications.

\noindent\textbf{Class Order Robustness.}
Our method demonstrates strong robustness to class order variations, which is a critical aspect in continual learning scenarios. This robustness stems from two key factors: First, the closed-form solution ensures deterministic and unique classification head weights for a given set of training data. As evidenced in Table \ref{table-VOC2012:1}, our method achieves identical performance in both the sequential 15-1 and 15-5 settings. This consistency is a direct consequence of the closed-form nature of our solution, which guarantees the same optimal weights regardless of the training sequence. Second, for 3D datasets, we observe that performance variations across different class orders are primarily influenced by the backbone network's feature extraction capabilities. When using the same backbone architecture, the performance remains remarkably stable across different class sequences. Minor variations in performance can be attributed to the backbone's sensitivity to different class orders during the initial training phase, rather than the continual learning mechanism itself. This robustness to class order variations is particularly valuable in real-world applications where the sequence of class introduction cannot be predetermined. Our method's ability to maintain consistent performance regardless of the learning order makes it more practical for deployment in dynamic environments.

\subsection{Ablation Studies}
To rigorously evaluate the contributions of each component in our framework, we conducted comprehensive ablation experiments using the challenging Pascal VOC2012 overlapped 10-1 setting. The results presented in Table \ref{tab:ablation} reveal several key insights into the effectiveness of our approach.

\noindent\textbf{Effect of RHL.}
First, it is evident from the table that removing the RHL component led to a significant performance drop, particularly on new classes (mIoU decreased from 41.20\% to 9.36\%), while the overall performance also sharply declined from 58.91\% to 37.94\%. This result underscores the critical role of RHL in enhancing the model's plasticity while maintaining stability for previously learned classes. By mapping features into a higher-dimensional space, RHL makes them more linearly separable, thereby improving the model's ability to learn new categories.

\noindent\textbf{Effect of Pseudo-Labeling.}
Second, our analysis demonstrates that the pseudo-labeling mechanism plays a vital role in preventing semantic drift, which is particularly evident in incremental learning scenarios with both disjoint and overlapped classes. Without the pseudo-labeling mechanism, the model's performance on old classes dropped from 75.02\% to 71.83\%, and on new classes from 41.20\% to 36.19\%, with the overall performance decreasing to 54.86\%. This indicates that pseudo-labeling effectively mitigates representation shifts, helping the model retain accurate recognition capabilities for previously learned classes.  

These experimental results empirically validate our theoretical framework, confirming that each architectural component addresses specific challenges in continual semantic segmentation, collectively contributing to the robust performance observed in our comprehensive experiments.

\begin{table}[htbp]
    \centering
    \caption{"Ours" represents the complete model setting; "w/o RHL" means the model excludes RHL; "w/o Pseudo-Labeling" indicates that the model does not use the pseudo-labeling.}
    \begin{tabular}{lccc}
        \toprule
        Settings & 0 - 10 & 11 - 20 & all \\
        \midrule
        Ours & \textbf{75.02} & \textbf{41.20} & \textbf{58.91} \\
        w/o RHL & 63.91 & 9.36 &  37.94 \\
        w/o Pseudo-Labeling & 71.83 & 36.19 & 54.86 \\
        \bottomrule
    \end{tabular}
    \label{tab:ablation}
\end{table}

\subsection{Efficiency Studies}
As analyzed in Sec \ref{efficiency}, our method demonstrates remarkable computational efficiency advantages over traditional gradient-based optimization approaches. Updating the classification head parameters only requires matrix multiplication and matrix inversion, which can be efficiently parallelized on GPU. As quantified in Table \ref{tab:performance_comparison}, our closed-form solution achieves convergence in just a single training epoch, taking only 43.25 seconds, whereas fine-tuning methods requiring multiple gradient descent iterations need 651.46 seconds. This represents a 15× acceleration in training time while maintaining superior segmentation performance.  Furthermore, despite supporting larger batch sizes (64 vs. 32), our method significantly reduces memory consumption (51.61 GB vs. 59.55 GB) due to the closed-form nature of our solution, which eliminates the need to store historical data or intermediate gradient states. These efficiency metrics highlight the practical advantages of our approach for resource-constrained deployment scenarios and time-sensitive applications.

\begin{table}[htbp]
\caption{Comparison of Training Time and GPU Memory, with fine-tuning for 10 epochs on the Pascal VOC2012 dataset.}

\centering
    \resizebox{\linewidth}{!}{
\begin{tabular}{lcccc}
\toprule
\textbf{Method}     & \textbf{One epoch Time} & \textbf{Total Time} & \textbf{GPU Memory} & \textbf{Batch Size} \\
\midrule
FT          & 64.39 s                  & 651.46 s                           & 59.55 GB                     & 32                      \\
Ours & 43.25 s                 & 43.25 s                             & 51.61 GB                     & 64                      \\
\bottomrule
\end{tabular}}
\label{tab:performance_comparison}
\end{table}
\section{Discussion}
\subsection{Scalability Challenges in Large Datasets}

Although our method demonstrates excellent performance across the tested datasets, its scalability to extremely large datasets may face challenges. The method relies on a pre-trained encoder that is frozen after the initial step. For very complex or diverse data distributions, this frozen encoder might not capture sufficiently rich features to support effective classification across all future tasks.  Nevertheless, this limitation could be mitigated by leveraging more robust pre-trained backbone networks, such as SAM \cite{kirillov2023segment}, which have been pre-trained on large-scale  datasets.

\subsection{Application to Real-Time Systems}
Our method is particularly well-suited for real-time systems and edge computing scenarios due to several key advantages. First, its closed-form solution approach requires only a single update per step, significantly reducing computational demands compared to gradient-based methods that require multiple epochs, as demonstrated in our experiments. Second, the minimal memory footprint of the method, which eliminates the need to store historical data, makes it ideal for resource-constrained devices. Third, the deterministic nature of the closed-form solution ensures consistent performance, avoiding the variability inherent in stochastic gradient-based approaches. These characteristics make our method an excellent choice for applications such as autonomous vehicles, where continuous learning from streaming data must operate under strict latency and resource constraints.

\section{Conclusion}
We presented CFSSeg, a novel method for class-incremental semantic segmentation (CSS) designed for both 2D images and 3D point clouds. CFSSeg distinguishes itself through a gradient-free, closed-form update mechanism, computed recursively for efficiency. This core component, combined with a frozen encoder for stability, high-dimensional feature mapping for plasticity, and a tailored pseudo-labeling strategy for semantic drift, allows the model to learn new classes incrementally without catastrophic forgetting or reliance on stored exemplars. Consequently, CFSSeg operates with significantly reduced computational cost—requiring only a single training pass per step—and enhanced data privacy. Our extensive evaluations on Pascal VOC2012, S3DIS, and ScanNet show that CFSSeg achieves outstanding results, outperforming prior methods and providing a robust, efficient, and effective solution for continual semantic segmentation tasks, forming a complete closed loop from theoretical foundation to practical implementation.


\begingroup
\small
\bibliographystyle{IEEEtran}
\bibliography{icme2025references}
\endgroup

\onecolumn

\appendices
\section{Proof of Theorem \ref{theorem:main}}
\label{app:proof_of_the_theorem}
\begin{proof}

At step $t-1$, we have
\begin{align}
\label{eq_w_k_1}
\mathbf{\hat \Phi}_{t-1} = (\mathbf{E}_{1:t-2}^{\top} \mathbf{E}_{1:t-2}+\mathbf{E}_{t-1}^{\top}\mathbf{E}_{t-1}+\gamma \mathbf{I})^{-1} \begin{bmatrix} \mathbf E_{1:t-1}^{\top} \mathbf Y_{1:t-2}^{\text{train}}+\mathbf E_{t-1}^{\top}\mathbf {\bar Y}_{t-1}^{\text{train}} & \mathbf E_{t-1}^{\top} \mathbf {\tilde{Y}}_{t-1}^{\text{train}}
\end{bmatrix}.
\end{align}

Hence, at step $t$, we have
\begin{align}
\label{eq_w_k}
\mathbf{\hat \Phi}_{t} = (\mathbf{E}_{1:t-1}^{\top} \mathbf{E}_{1:t-1}  +\mathbf{E}_{t}^{\top}\mathbf{E}_{t}+\gamma \mathbf{I})^{-1} \begin{bmatrix} \mathbf E_{1:t-1}^{\top} \mathbf Y_{1:t-1}^{\text{train}}+\mathbf E_{t}^{\top}\mathbf {\bar Y}_{t}^{\text{train}} & \mathbf E_{t}^{\top} \mathbf {\tilde{Y}}_{t}^{\text{train}} 
\end{bmatrix}.
\end{align}

We have defined the autocorrelation memory matrix $\mathbf{\Psi}_{t-1}$ in the paper via
\begin{align}\label{eq_r_m_k_1}
    \mathbf{\Psi}_{t-1} = (\mathbf{E}_{1:t-2}^{\top} \mathbf{E}_{1:t-2}   + \mathbf{E}_{t-1}^{\top}\mathbf{E}_{t-1}+\gamma \mathbf{I})^{-1}.
\end{align}

To facilitate subsequent calculations, here we also define a cross-correlation matrix $\mathbf{Q}_{t-1}$, i.e., 
\begin{align}\label{eq_Q_k_1}
    \mathbf{Q}_{t-1} = \begin{bmatrix} \mathbf E_{1:t-1}^{\top} \mathbf Y_{1:t-2}^{\text{train}}+\mathbf E_{t-1}^{\top}\mathbf {\bar Y}_{t-1}^{\text{train}} & \mathbf E_{t-1}^{\top}  \mathbf {\tilde{Y}}_{t}^{\text{train}}
\end{bmatrix}.
\end{align}

Thus we can rewrite Eqn \eqref{eq_w_k_1} as
\begin{align}\label{xx}
    \hspace{4pt}\mathbf{\hat \Phi}_{t-1} = \mathbf{\Psi}_{t-1}\mathbf{Q}_{t-1}.
\end{align}

Therefore, at step ${t}$ we have
\begin{align}\label{compact}
    \mathbf{\hat \Phi}_{t} = \mathbf{\Psi}_{t}\mathbf{Q}_{t}.
\end{align}

From Eqn \eqref{eq_r_m_k_1}, we can recursively calculate $\mathbf{\Psi}_{t}$ from $\mathbf{\Psi}_{t-1}$, i.e.,
\begin{align}\label{eq_r_m_k}
    \mathbf{\Psi}_{t} = 
    (\mathbf{\Psi}_{t-1}^{-1} + \mathbf{E}_{t}^{\top}\mathbf{E}_{t})^{-1}.        
\end{align}
    
According to the Woodbury matrix identity, we have
\begin{align*}
    (\mathbf{A} + \mathbf{U}\mathbf{C}\mathbf{V})^{-1} = \mathbf{A}^{-1} - \mathbf{A}^{-1}\mathbf{U}(\mathbf{C}^{-1} + \mathbf{V}\mathbf{A}^{-1}\mathbf{U})^{-1}\mathbf{V}\mathbf{A}^{-1}.
\end{align*}
  
Let $\mathbf{A} = \mathbf{\Psi}_{t-1}^{-1}$, $\mathbf{U} = \mathbf{E}_{t}^{\top}$, $\mathbf{C} = \mathbf{I}$, and $\mathbf{V} = \mathbf{E}_{t}$ in Eqn \eqref{eq_r_m_k}, we have
\begin{align}\label{eq_R_update1}
    \mathbf{\Psi}_{t} = \mathbf{\Psi}_{t-1} - \mathbf{\Psi}_{t-1}\mathbf{E}_{t}^{\top}(\mathbf{I} + \mathbf{E}_{t}\mathbf{\Psi}_{t-1}\mathbf{E}_{t}^{\top})^{-1}\mathbf{E}_{t}\mathbf{\Psi}_{t-1}.
\end{align}
  
Hence, $\mathbf{\Psi}_{t}$ can be recursively updated using its last-phase counterpart $\mathbf{\Psi}_{t-1}$ and data from the current phase (i.e., $\mathbf{E}_{t}$). This proves the recursive calculation of the inverted auto-correlation matrix.

Next, we derive the recursive formulation of $\mathbf{\hat \Phi}_{t}$. To this end, we also recurse the cross-correlation matrix $\mathbf{Q}_{t}$ at step $t$, i.e.,
\begin{align}\label{eq_R_update3}
\mathbf{Q}_{t} = \begin{bmatrix} \mathbf E_{1:t-1}^{\top} \mathbf Y_{1:t-1}^{\text{train}}+\mathbf E_{t}^{\top}\mathbf {\bar Y}_{t}^{\text{train}} & \mathbf E_{t}^{\top} \mathbf {\tilde{Y}}_{t}^{\text{train}} \end{bmatrix} = \mathbf{Q}_{t-1}^{\prime}+ \begin{bmatrix}
    \mathbf{E}_{t}^{\top}\mathbf {\bar Y}_{t}^{\text{train}} & \mathbf {E}_{t}^{\top} \mathbf{\tilde{Y}}_{t}^{\text{train}}
\end{bmatrix},
\end{align}
where
\begin{align}\label{Qt-1}
\mathbf{Q}_{t-1}^{\prime} & = \begin{cases}
\begin{bmatrix}
    \mathbf{Q}_{t-1} & \mathbf{0}_{d_{\text{E}}\times (d_{C_{t}}-d_{C_{t-1}})}\end{bmatrix}, &d_{C_{t}} > d_{C_{t-1}} \\
    \mathbf{Q}_{t-1}, &d_{C_{t}} = d_{C_{t-1}} 
\end{cases} .
\end{align}

Note that the concatenation in Eqn \eqref{Qt-1} is due to the assumption that $\mathbf{Y}_{1:t}^{\text{train}}$ at step $t$ contains more data classes (hence more columns) than $\mathbf{Y}_{1:t-1}^{\text{train}}$. It is possible that there are no new classes appear at step $t$, then $ \mathbf {\tilde{Y}}_{t}^{\text{train}}$ should be $\mathbf{0}$.

Similar to what Eqn \eqref{Qt-1} does,
\begin{align}
\mathbf{\hat \Phi}_{(t-1)\prime} & = \begin{cases}
\begin{bmatrix}\mathbf{\hat \Phi}_{t-1} & \mathbf{0}_{d_{\text{E}}\times (d_{C_{t}}-d_{C_{t-1}})}\end{bmatrix}, & d_{C_{t}} > d_{C_{t-1}} \\
 \mathbf{\hat \Phi}_{t-1}, & d_{C_{t}} = d_{C_{t-1}} 
\end{cases}
\end{align}

We have
\begin{align}\label{wfcn}
    \mathbf{\hat \Phi}_{(t-1)\prime} = \mathbf{\Psi}_{t-1}\mathbf{Q}_{t-1}^{\prime}.
\end{align}
Hence, $\mathbf{\hat \Phi}_{t}$ can be rewritten as
\begin{align}\nonumber
    \mathbf{\hat \Phi}_{t} &= \mathbf{\Psi}_{t}\mathbf{Q}_{t} \\ \nonumber
    &=    \mathbf{\Psi}_{t}(\mathbf{Q}_{t-1}^{\prime} + \begin{bmatrix}
    \mathbf{E}_{t}^{\top}\mathbf {\bar Y}_{t}^{\text{train}} & \mathbf E_{t}^{\top} \mathbf {\tilde{Y}}_{t}^{\text{train}} 
\end{bmatrix})\\ \label{eq_W_k_33}
    &=\mathbf{\Psi}_{t}\mathbf{Q}_{t-1}^{\prime} + \mathbf{\Psi}_{t}\mathbf{E}_{t}^{\top}\begin{bmatrix}
    \mathbf {\bar Y}_{t}^{\text{train}} & \mathbf {\tilde{Y}}_{t}^{\text{train}} 
\end{bmatrix}.
\end{align}

By substituting Eqn \eqref{eq_R_update1} into $\mathbf{\Psi}_{t}\mathbf{Q}_{t-1}^{\prime}$, we have
\begin{align}\nonumber
\mathbf{\Psi}_{t}\mathbf{Q}_{t-1}^{\prime} &=  \mathbf{\Psi}_{t-1}\mathbf{Q}_{t-1}^{\prime} - \mathbf{\Psi}_{t-1}\mathbf{E}_{t}^{\top}(\mathbf{I} + \mathbf{E}_{t}\mathbf{\Psi}_{t-1}\mathbf{E}_{t}^{\top})^{-1}\mathbf{E}_{t}\mathbf{\Psi}_{t-1}\mathbf{Q}_{t-1}^{\prime}\\ \label{eq_W_k_2}
&=\mathbf{\hat \Phi}_{(t-1)\prime}- \mathbf{\Psi}_{t-1}\mathbf{E}_{t}^{\top}(\mathbf{I} + \mathbf{E}_{t}\mathbf{\Psi}_{t-1}\mathbf{E}_{t}^{\top})^{-1}\mathbf{E}_{t}\mathbf{\hat \Phi}_{(t-1)\prime}.
\end{align}

To simplify this equation, let $\mathbf{K}_{t} = (\mathbf{I} + \mathbf{E}_{t}\mathbf{\Psi}_{t-1}\mathbf{E}_{t}^{\top})^{-1}$. Since 
\begin{align*}
    \mathbf{I} = \mathbf{K}_{t}\mathbf{K}_{t}^{-1} = \mathbf{K}_{t}(\mathbf{I} + \mathbf{E}_{t}\mathbf{\Psi}_{t-1}\mathbf{E}_{t}^{\top}),
\end{align*}
we have
$\mathbf{K}_{t} = \mathbf{I} - \mathbf{K}_{t} \mathbf{E}_{t}\mathbf{\Psi}_{t-1}\mathbf{E}_{t}^{\top}$.
Therefore, 
\begin{align}\label{a}
    &\mathbf{\Psi}_{t-1}\mathbf{E}_{t}^{\top}(\mathbf{I} + \mathbf{E}_{t}\mathbf{\Psi}_{t-1}\mathbf{E}_{t}^{\top})^{-1} \notag \\
    &= \mathbf{\Psi}_{t-1}\mathbf{E}_{t}^{\top}\mathbf{K}_{t}\notag \\
    &=\mathbf{\Psi}_{t-1}\mathbf{E}_{t}^{\top}(\mathbf{I} - \mathbf{K}_{t} \mathbf{E}_{t}\mathbf{\Psi}_{t-1}\mathbf{E}_{t}^{\top})\notag \\
    &=(\mathbf{\Psi}_{t-1} - \mathbf{\Psi}_{t-1}\mathbf{E}_{t}^{\top}\mathbf{K}_{t} \mathbf{E}_{t}\mathbf{\Psi}_{t-1})\mathbf{E}_{t}^{\top} \notag \\
    &= \mathbf{\Psi}_{t}\mathbf{E}_{t}^{\top}.
\end{align}

Substituting Eqn \eqref{a} into Eqn \eqref{eq_W_k_2}, $\mathbf{\Psi}_{t}\mathbf{Q}_{t-1}^{\prime}$ can be written as
\begin{align}\label{eq_RQprime}
    \mathbf{\Psi}_{t}\mathbf{Q}_{t-1}^{\prime} = \mathbf{\hat \Phi}_{(t-1)\prime}- \mathbf{\Psi}_{t}\mathbf{E}_{t}^{\top}\mathbf{E}_{t}\mathbf{\hat \Psi}_{(t-1)\prime}.
\end{align}

Substituting Eqn \eqref{eq_RQprime} into Eqn \eqref{eq_W_k_33} implies that
\begin{align}\nonumber
    \mathbf{\hat \Phi}_{t} &= 
    \mathbf{\hat \Phi}_{(t-1)\prime}- \mathbf{\Psi}_{t}\mathbf{E}_{t}^{\top}\mathbf{E}_{t}\mathbf{\hat \Phi}_{(t-1)\prime}+\mathbf{\Psi}_{t}\mathbf{E}_{t}^{\top}\begin{bmatrix}
    \mathbf {\bar Y}_{t}^{\text{train}} & \mathbf {\tilde{Y}}_{t}^{\text{train}} \end{bmatrix}
    \\
   &  = \begin{bmatrix}
  \mathbf{\hat \Phi}_{t-1}- \mathbf{\Psi}_{t}\mathbf{E}_{t}^{\top}\mathbf{E}_{t}\mathbf{\hat \Phi}_{t-1}+\mathbf{\Psi}_{t}\mathbf E_{t}^{\top}\mathbf {\bar Y}_{t}^{\text{train}} &
\mathbf{\Psi}_{t}\mathbf E_{t}^{\top} \mathbf {\tilde{Y}}_{t}^{\text{train}} \end{bmatrix}
\end{align}
which completes the proof.

\end{proof}

\end{document}